\documentclass[runningheads,a4paper]{llncs}
\usepackage{amssymb}
\usepackage{mydefs}

\usepackage{url}
\urldef{\mailsa}\path|{alfred.hofmann, ursula.barth, ingrid.haas, frank.holzwarth,|
\urldef{\mailsb}\path|anna.kramer, leonie.kunz, christine.reiss, nicole.sator,|
\urldef{\mailsc}\path|erika.siebert-cole, peter.strasser, lncs}@springer.com|
\newcommand{\keywords}[1]{\par\addvspace\baselineskip
\noindent\keywordname\enspace\ignorespaces#1}

\begin{document}
\mainmatter  % start of an individual contribution

\title{A Geometric View of Conjugate Priors}

\titlerunning{A Geometric View of Conjugate Priors}

\author{Arvind Agarwal \and Hal Daum\'e III}
\authorrunning{Arvind Agarwal \and Hal Daum\'e III}

\institute{School of Computing,\\
University of Utah,\\
Salt Lake City, Utah, 84112 USA\\
\{arvind,hal\}@cs.utah.edu
}


\toctitle{A Geometric View of Conjugate Priors}
\tocauthor{Arvind Agarwal  Hal Daum\'e III}
\maketitle

\begin{abstract}
In Bayesian machine learning, conjugate priors are popular, mostly due to mathematical convenience. In this paper, we show that there are deeper reasons for choosing a conjugate prior. Specifically, we formulate the conjugate prior in the form of Bregman divergence and show that it is the inherent geometry of conjugate priors that makes them appropriate and intuitive. This geometric interpretation allows one to view the hyperparameters of conjugate priors as the {\it effective} sample points, thus providing additional intuition. We use this geometric understanding of conjugate priors to derive the hyperparameters and expression of the prior used to couple the generative and discriminative components of a hybrid model for semi-supervised learning.
\end{abstract}

\section{Introduction}
\label{sec:intro}
In probabilistic modeling, a practitioner typically chooses a likelihood function (model) based on her knowledge of the problem domain.  With limited training data, a simple maximum likelihood (ML) estimation of the parameters of this model will lead to overfitting and poor generalization.  One can regularize the model by adding a prior, but the fundamental question is: which prior?  We give a turn-key answer to this problem by analyzing the underlying \emph{geometry} of the likelihood model and suggest choosing the unique prior with the same geometry as the likelihood.  This unique prior turns out to be the \emph{conjugate} prior, in the case of the exponential family.  This provides justification beyond ``computational convenience'' for using the conjugate prior in machine learning and data mining applications.

% Model selection is a prevalent problem in machine learning. In most
% supervised learning problems, one is interested in minimizing the
% generalization error of a learned model. In probabilistic terms, this
% generalization ability is controlled through a prior on the model
% parameters. From a subjectivist perspective, this means that the user
% needs to have some prior knowledge about the model's parameters. In
% the absence of the strong prior knowledge, we still want to achieve
% good generalization, typically by encouraging ``simple'' model
% parameters. There are many ways of choosing a prior that enforces this
% simplicity; in practice, one typically always chooses the prior that
% is conjugate to data likelihood model\footnote{In Bayesian world, one
%   first chooses a model that is appropriate for the problem, and then
%   chooses the prior that will avoid the over-fitting i.e. in a text
%   classification problem, an appropriate likelihood model is {\it
%     Binomial model}, for which, the prior would be a {\it Beta
%     prior}.}. Frequently, the only reason for this choice is
% computational. This renders the subsequent inference problem simpler,
% and often avoids the need for approximate inference. In this paper,
% using the concepts borrowed from geometry, we give deeper reasons as
% to why it is appropriate to choose a conjugate prior. Mainly, we show
% that for distributions in the exponential family, it is natural to use
% conjugate priors because conjugate prior induces the same geometric
% structure as the likelihood.

In this work, we give a geometric understanding of maximum likelihood estimation method and a geometric argument in the favor of using conjugate priors. In \secref{likelihoodGeom}, first we formulate the ML estimation problem into a completely geometric problem with no explicit mention of probability distributions. We then show that this geometric problem carries a geometry that is inherent to the structure of the likelihood model. For reasons given in Sections~\ref{sec:conjugateIter} and \ref{sec:nonConjugateIter}, when considering the prior, it is important that one uses the same geometry as likelihood. Using the same geometry also gives the closed-form solution for the maximum-a-posteriori (MAP) problem. We then analyze the prior using concepts borrowed from the information geometry. We show that this geometry induces the {\it   Fisher information metric} and {\it 1-connection}, which are respectively, the natural metric and connection for the exponential family (Section~\ref{sec:infogeom}). One important outcome of this analysis is that it allows us to treat the hyperparameters of the conjugate prior as the effective sample points drawn from the distribution under consideration. This analysis also allows us to extend the results of MAP estimation in the exponential family to the $\alpha$-family (Section~\ref{sec:alphaFamily}) because, similar to exponential families, $\alpha$-families also carry an inherent geometry \cite{zhang04divegence}. We finally extend this geometric interpretation of conjugate priors to analyze the hybrid model given by \cite{Minka2006} in a purely geometric setting and justify the argument presented in \cite{arvind09hybrid} (i.e. a {\it coupling   prior} should be conjugate) using a much simpler analysis (Section~\ref{sec:hybrid}). Our analysis couples the discriminative and generative components of hybrid model using the Bregman divergence which reduces to the coupling prior given in \cite{arvind09hybrid}. This analysis avoids the {\it explicit} derivation of the hyperparameters, rather automatically gives the hyperparameters of the conjugate prior along with the expression.

%%%%%%%%%%%%%%%%%%%%%%%%%%%%%%%%%%%%%%%%%%%%%%%%%%%%%%%%%%%%%%%
\section{Motivation}
\label{sec:motiv}
Our analysis is driven by the desire to understand the geometry of the conjugate priors for the exponential families. This understanding has many advantages that are described in the remainder of the paper: an extension of notion of conjugacy beyond the exponential family (to $\alpha$-family),
and geometric analysis of models that use the conjugate priors \cite{arvind09hybrid}.

We motivate our analysis by asking ourselves the following question: Given a parametric model $p(x;\th)$ for the data likelihood, and a prior on its parameters $\th$, $p(\th;\alpha,\beta)$; what should the hyperparameters $\alpha$ and $\beta$ of the prior encode? We know that $\th$ in the likelihood model is the estimation of the parameter using the given data points. In other words, the estimated parameter fits the model according to the given data  while prior on the parameter provides the generalization.
%HAL: I disagree with the above line! - corrected
 This generalization is enforced by some prior belief encoded in the hyperparameters. Unfortunately, one does not know what is the likely value of the parameters; rather one might have some belief in what \emph{data points} are likely to be sampled from the model. Now the question is: Do the hyperparameters encode this belief in the parameters in terms of the sampling points? Our analysis shows that the hyperparameters of the conjugate prior is nothing but the effective sampling points. In case of non-conjugate priors, interpretation of hyperparameters is not clear.

A second motivation is the following geometric analysis.  Before we go into the problem, consider two points in the {\it Euclidean} space which one would like to interpolate using a parameter $\gamma \in [0,1]$. A natural way to do so is to interpolate them linearly i.e., connect two points using a straight line, and then find the interpolating point at the desired $\gamma$, as shown in \figref{euc}. This interpolation scheme does not change if we move to a non-Euclidean space.  In other words, if we were to interpolate two points in the non-Euclidean space, we would find the interpolating point by connecting the two points by a geodesic (an equivalent to the straight line in the non-Euclidean space) 
%HAL: define "straight"! - corrected
and then finding the point at the desired $\gamma$, shown in \figref{noneuc}. 
%HAL: you should reference the figure in this paragraph... should the lines in (b) actually "look" straight? - corrected

\begin{figure}[tc]
\begin{center}
\subfigure[]{\label{fig:euc}\includegraphics[scale=0.45]{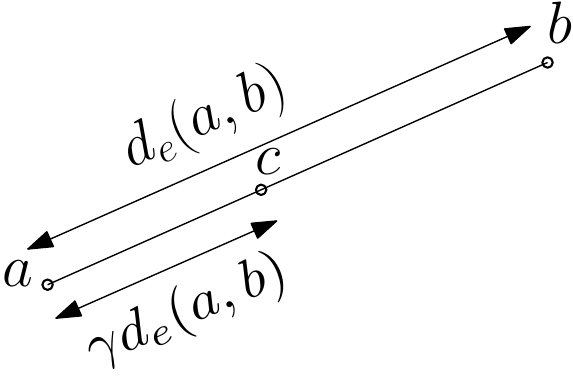}} \hspace{1in}
\subfigure[]{\label{fig:noneuc}\includegraphics[scale=0.40]{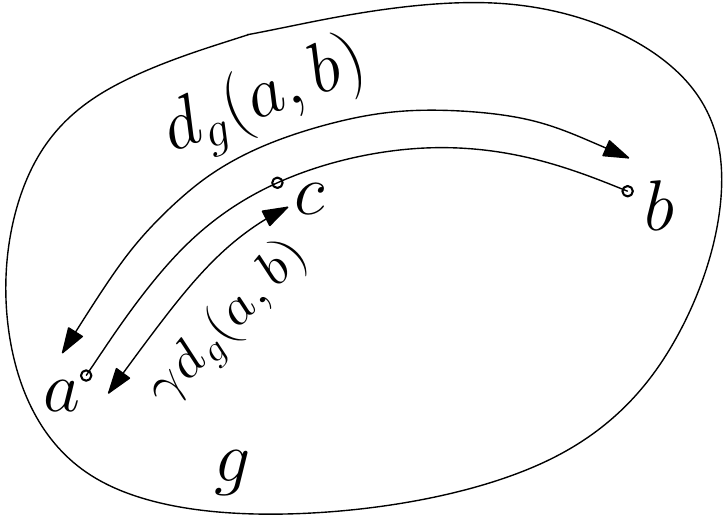}}
\caption{{\small\textsf{ Interpolation of two points $a$ and $b$ using (a) Euclidean geometry, and (b) non-Euclidean geometry. Here geometry is defined by the respective distance/divergence functions $d_e$ and $d_g$. It is important to notice that the divergence is a generalized notion of the distance in the non-Euclidean spaces, in particular, in the spaces of the exponential family statistical manifolds. In these spaces, it is the divergence function that define the geometry.}}} %Arvind: read if caption makes sense
\label{fig:eucAndNoneuc}
\end{center}
\end{figure}

This situation arises when one has two models and wants to build a better model by interpolating them. This exact situation is
encountered in \cite{Minka2006} where the objective is to build a hybrid model by interpolating (or coupling) discriminative and generative models. Agarwal et.al. \cite{arvind09hybrid} couples these two models using the conjugate prior, and empirically shows using a
conjugate prior for the coupling outperforms the original choice \cite{Minka2006} of a Gaussian prior.  In this work, we find the hybrid model by
interpolating the two models using the {\it inherent geometry}\footnote{In exponential family statistical manifold, inherent geometry is defined by the divergence function because it is the divergence function that induces the metric structure and connection of the manifold. Refer \cite{amarinagaoka} for more details.} %Arvind:is this footnote ok?
of the space (interpolate along the geodesic in the space defined by the inherent geometry) which automatically results in the conjugate prior along with its hyperparameters. Our analysis and the analysis of Agarwal et al. lead to the same result, but ours is much simpler and
naturally extends to the cases where one wants to couple more than two models. One big advantage of our analysis is that unlike prior
approaches \cite{arvind09hybrid}, we need not know the expression and the hyperparameters of the prior in advance. They are automatically derived by the analysis. Our analysis based on the geometric interpretation can also be used to interpolate the models using a polynomial of higher degree instead of just the straight line i.e., quadratic interpolation etc., and to derive the corresponding prior. Our analysis only requires the inherent geometry which is given by the models under the consideration and the interpolation parameters (parameters of the polynomial). No explicit expression of the coupling
prior is needed.

%%%%%%%%%%%%%%%%%%%%%%%%%%%%%%%%%%%%%%%%%%%%%%%%%%%%%%%%%%%%%%%
\section{Background}
\label{sec:background}
In this section. we give the required background, specially, we revisit the concepts related to Legendre duality, exponential families and Bregman divergence.

\subsection{Legendre Duality}
Let $\c{M} \subseteq \b{R}^d$ and $\Theta \subseteq \b{R}^d$ be two
spaces and let $F:\c{M} \to \b{R}^+$ and $G:\Theta \to \b{R}^+$ be two
convex functions. $F$ and $G$ are said to be \emph{conjugate duals} of each other if:
\begin{align}
F(\mu) := \sup_{\th \in \Theta} \{\innerprod{\mu,\th} - G(\th)\}
\label{conjugatedual}
\end{align}
here $\langle a,b \rangle$ denotes the dot product of vectors $a$ and $b$. The spaces ($\Theta$ and $\c{M}$) associated with these dual functions are called {\it dual spaces}. We sometime use the standard notation to refer this duality i.e., $ G = F^*$ and $F = G^*$. A particularly important connection between dual spaces is that: for each $\mu \in \c{M}, \grad F(\mu) = \th \in \Th$ (denoted as $\mu^* = \th)$) and similarly, for each $\th \in \Th, \grad G(\th) = \mu \in \c{M}$ (or $\th^* = \mu)$).  For more details, refer to \cite{rockafellarConvex}. \figref{duality} gives a pictorial representation of this duality and the notations associated with it.

\begin{figure}[tc]
 \begin{center}
  \includegraphics[width=0.25 \textwidth]{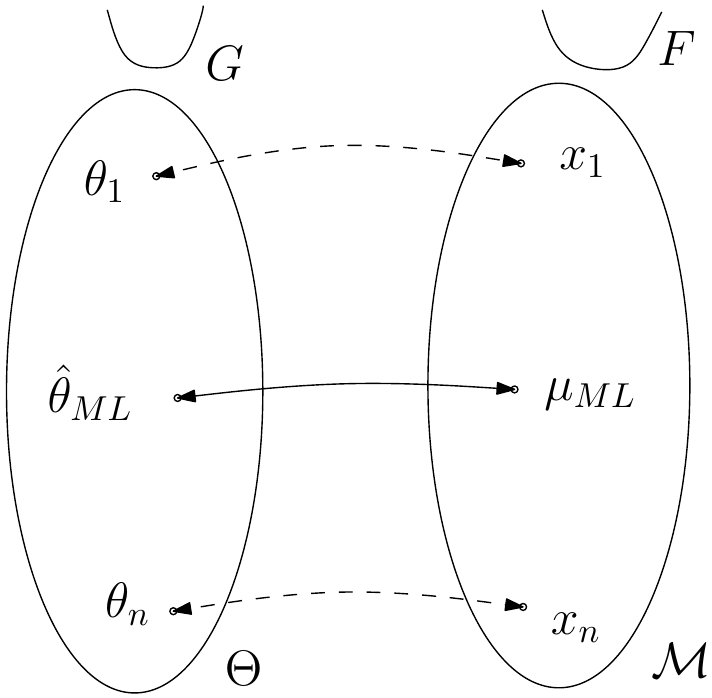}
  \caption{{\small\textsf{Duality between mean parameters and natural parameters. Notice the convex functions defined over both spaces. these functions are dual of each other and so are the spaces.}}}
  \label{fig:duality}
 \end{center}
\end{figure}

\subsection{Bregman Divergence}
We now give a brief overview of Bregman divergence (for more details see \cite{Banerjee05Clustering}). Let  $F: \c{M} \to \b{R}$ be a continuously-differentiable real-valued and strictly {\it convex function} defined on a closed {\it convex set} $\c{M}$. The Bregman divergence associated with $F$ for points $p, q \in \c{M}$ is:
{\small
\begin{equation}
B_F(p \rVert q) = F(p)-F(q)-\langle \grad F(q),(p-q)\rangle
\label{bregman}
\end{equation}
}

\noindent
If $G$ is the {\it conjugate dual} of $F$ then:
{\small
\begin{align}
B_F(p\|q) = B_{G}(q^*\|p^*)
\label{bregmanflip}
\end{align}
}
here $p^*$ and $q^*$ are the duals of $p$ and $q$ respectively. It is emphasized that Bregman divergence is not symmetric i.e., in general, $B_F(p\|q) \ne B_F(q\|p)$, therefore it is important what directions these divergences are measured in.

\subsection{Exponential Family}
In this section, we review the exponential family. The exponential family is a set of distributions, whose probability density function\footnote{``Density function'' can be replaced by ``mass function'' in the case of discrete random variables} can be expressed in the following form:
\begin{equation}
p(x;\th)=p_o(x)\exp(\innerprod{\theta, \phi(x)}-G(\th))
\label{expdist}
\end{equation}
here $\phi(x):\c{X}^m \to \b{R}^d$ is a vector \emph{potentials}  or \emph{sufficient statistics} and $G(\theta)$ is a normalization constant or {\it log-partition function} which ensures that distributions are normalized. With the potential functions $\phi(x)$ fixed, every $\th$ induces a particular member $p(x;\th)$ of the family. In our framework, we deal with exponential families that are \emph{regular} and have the \emph{minimal representation}\cite{VariationExpoFamilies}.

The exponential family has a number of convenient properties and subsumes many common distributions.  It includes the Gaussian, Binomial, Beta, Multinomial and Dirichlet distributions, hidden Markov models, Bayes nets, etc.. One important property of the exponential family is the existence of conjugate priors.  Given any member of the exponential family in \eqref{expdist}, the \emph{conjugate prior} is a distribution over its \emph{parameters} with the following form:
\begin{equation}
p(\th|\alpha,\beta)=m(\alpha,\beta) \;\exp(\langle \th, \alpha \rangle - \beta G(\th))
\label{conjprior}
\end{equation}
here $\alpha$ and $\beta$ are hyperparameters of the conjugate prior.  Importantly, the function $G(\cdot)$ is the same between the exponential family member and its conjugate prior.

A second important property of exponential family member is that log-partition function $G$ is convex and defined over the convex set $\Th := \{\th \in \b{R}^d : G(\th) < \infty\}$.  Since the log-partition function $G$ is convex over this set, it induces a Bregman divergence on the space $\Theta$.

Another important property of the exponential family is the {\it one-to-one} mapping between the \emph{canonical parameters} $\th$ and the so-called ``\emph{mean parameters}'' which we denote by $\mu$. For each canonical parameter $\th \in \Theta$, there exists a mean parameter $\mu$, which belongs to the space $\c{M}$ defined as:
\begin{align}
 \c{M} := \Big\{ \mu \in \b{R}^d : \mu = \int \phi(x) p(x;\th)  \,dx \quad \forall \th \in \Th \Big\}
\label{muspace}
\end{align}
Our notation has been deliberately suggestive.  $\Theta$ and $\c{M}$ are dual spaces, in the sense of Legendre duality because of the following relationship between the log-partition function $G(\th)$ and the expected value of the sufficient statistics $\phi(x)$:
\begin{align}
\label{expfamnorm}
\grad G(\th) = \b{E}(\phi(x)) = \mu.
\end{align}
In Legendre duality, we know that two spaces $\Th$ and $\c{M}$ are dual of each other if for each $\th \in \Th, \, \grad G(\th) = \mu \in \c{M}$. Here $G$ (the log partition function of the exponential family distribution) is the function defined on the space $\Th$. WE call the function in the dual space $\c{M}$ to be $F$ i.e., $F = G^*$. A pictorial representation of the duality between canonical parameter space $\Theta$ and mean parameter space $\c{M}$ is given in \figref{duality}.
%HAL: what I edited here is correct, right? - modified

%Another important relationship holds between the log-partition function $G(\th)$ and the sufficient statistics $\phi(x)$.  In particular, for each $\th \in \Th$, $\grad G(\th) = \th^*$, where $\th^* = u$ is the dual point to $\th$. Let  $\c{U}$ be the dual space such that 
%\begin{align}
%\c{U} = \Big\{ u \in \b{R}^d : u = \grad G(\th),\; \forall \, \th \in \Th \Big\}
%\label{uspace} 
%\end{align}
%It can be easily shown that for an exponential family $p(x;\th)$ with the log-partition function $G(\th)$ that each $u \in \c{U}$ is expected value of the sufficient statistic computed using the probability distribution defined by $\th \in \Theta$
%HAL: what I edited here is correct, right?
%\begin{align}
%\label{expfamnorm}
%\grad G(\th) = u = \mu
%\end{align}
%In other words, the space defined by \eqref{muspace} and the space given by \eqref{uspace} are same. From now onwards, we will use $\c{M}$ to refer both spaces. 
%%%%%%%%%%%%%%%%%%%%%%%%%%%%%%%%%%%%%%%%%%%%%%%%%%%%%%%%%%%%%%%

\section{Likelihood, Prior and Geometry}
\label{sec:mainsecgeometry}
In this section, we first formulate the ML problem into a Bregman median problem (Section~\ref{sec:likelihoodGeom}) and then show that corresponding MAP
problem can also be converted into a Bregman median problem (Section~\ref{sec:conjugateIter}). The MAP Bregman median problem consists of two parts: a likelihood model and a prior. We argue (Sections~\ref{sec:conjugateIter} and \ref{sec:nonConjugateIter}) that a Bregman median problem makes sense
%\footnote{In a median problem, it does not make sense to measure the distance between one set of points using one distance function and other set of points using another distance function} 
only when both of these parts have the same geometry.  Having the same geometry amounts to having the same log-partition function leading to the property of conjugate priors.

\subsection{Likelihood in the form of Bregman Divergence}
\label{sec:likelihoodGeom}
Following \cite{Collins01}, we can write the distributions belonging to the exponential family in terms of Bregman divergence. Let $p(x;\th)$ be the exponential family distribution as defined in \eqref{expdist}, the log of which (likelihood) can be written as\footnote{For the simplicity of the notations we will use $x$ instead of $\phi(x)$ assuming that $x \in \b{R}^d$. This does not change the analysis}:
\begin{equation}
\log p(x;\th)= \log p_o(x) + F(x) - B_F(x \rVert \grad G(\th))
\label{expbreg}
\end{equation}
This relationship depends on two observations: $F(\grad G(\th)) + G(\th) = \grad G(\th)\th $ and $\grad F(\th) = (\grad G)^{-1}(\th) \Rightarrow (\grad F)(\grad G(\th)) = \th $. These two observations can be used with \eqref{bregman} to see that \eqref{expbreg} is equivalent to the probability distribution defined in \eqref{expdist}. This representation of likelihood in the form of Bregman divergence gives insight in the geometry of the likelihood function. Gaining the insight into the exponential family distributions and establishing a meaningful relationship between likelihood and prior is the primary objective of this work.

In learning problems, one is interested in estimating the parameters $\th$ of the model which results in low generalization error. Perhaps the most standard estimation method is {\it maximum likelihood} (ML). The ML estimate, $\hat{\th}_{ML}$,  of a set of $n$ i.i.d. training data points $\c{X} = \{x_1,\ldots x_n\}$  drawn from the exponential family is obtained by solving the following problem:
$$ \hat{\th}_{ML} = \max_{\th \in \Th} \; \log p(\c{X};\th)$$

\begin{theorem}
 Let $\c{X} = \{x_1,\ldots x_n\}$ be a set of $n$ i.i.d. training data points drawn from the exponential family distribution with the log partition function $G$, $F$ be the dual function of $G$, then dual of ML estimate ($\hat{\th}_{ML}$) of $\c{X}$ under the assumed exponential family model solves the following Bregman median problem:
$$ \hat{\mu}_{ML} = \min_{\mu \in \c{M}} \sum_{i=1}^n B_F(x_i \rVert \mu)$$
\label{mediantheorem1}
\end{theorem}
\begin{proof}
The log-likelihood of $\c{X}$ under the assumed exponential family distribution is given by $\log p(\c{X};\th) = \sum_{i=1}^n \log p(x_i;\th)$ which along with \eqref{expbreg} can be used to compute $\hat{\th}_{ML}$:
{\small
\begin{align}
\hat{\th}_{ML} & = \max_{\th \in \Th}\sum_{i=1}^n \Big( \log p_o(x_i) + F(x_i) - B_F(x_i \rVert \grad G(\th)) \Big) \nonumber \\
	   & = \min_{\th \in \Th} \sum_{i=1}^n B_F(x_i \rVert \grad G(\th))
\label{mlmeanspace}
\end{align}
}
which using \eqref{expfamnorm} gives the desired result.
\end{proof}
The above theorem converts the problem of maximizing the log likelihood $\log p(\c{X};\th)$ into an equivalent problem of minimizing the corresponding Bregman divergences which is nothing but a \emph{Bregman median} problem, the solution to which is given by $\hat{\mu}_{ML} = \sum_{i=1}^n x_i$. ML estimate $\hat{\th}_{ML}$ can now be computed using \eqref{expfamnorm}, $ \hat{\th}_{ML} = (\grad G)^{-1}(\hat{\mu}_{ML})$.

\begin{lemma}
 If $x$ is the sufficient statistics of the exponential family with the log partition function $G$, and $F$ is the dual function of $G$ defined over the mean parameter space $\c{M}$  then $x \in \c{M}$.
\label{muxlemma}
\end{lemma}
\begin{proof}
 Refer \cite{kassvos} pg 39.
\end{proof}

\begin{lemma}
 Let $x$ be the sufficient statistics of the exponential family with the log partition function $G$, and $F$ be the dual function of $G$ defined over the mean parameter space $\c{M}$,  then there exists a $\th \in \Theta$, such that $x^*=\th$.
\label{thetaxlemma}
\end{lemma}
\begin{proof}
 $\c{M}$ and $\Theta$ are dual of each other so by the definition of duality, for every $\mu \in \c{M}$, there exists a $\th \in \Theta$ such that $\th = \mu^*$, and From Lemma~\ref{muxlemma} since  $x\in \c{M}$, which implies $x^* = \theta$.
\end{proof}

\begin{corollary}[ML as Bregman Median]
 Let $G(\th)$ be the log partition function of the exponential family defined over the convex set $\Theta$, $\c{X} = \{x_1,\ldots x_n\}$ be set of $n$ i.i.d data points drawn from this exponential family, and $\th_i$ be the dual of $x_i$, then ML estimation, $\hat{\th}_{ML}$ of $\c{X} = \{x_1,\ldots x_n\}$ solves the following optimization problem:
\begin{align}
 \hat{\th}_{ML} = \min_{\th \in \Th} \sum_{i=1}^n B_G(\th \| \th_i)
\label{mlindual}
\end{align}
\end{corollary}
\begin{proof}
 Proof directly follows from Lemma~\ref{thetaxlemma} and Theorem~\ref{mediantheorem1}. From Lemma~\ref{thetaxlemma}, we know that $x_i^* = \th_i$. Now using Theorem~\ref{mediantheorem1} and  \eqref{bregmanflip}, $B_F(x_i \rVert \mu) = B_G(\th \rVert x_i^*) = B_G(\th \rVert \th_i)$.
\end{proof}

The above expression requires us to find a $\th$ so that divergence from $\th$ to other $\th_i$ is minimized. Now note that $G$ is what defines this divergence and hence the geometry of the $\Theta$ space (as discussed earlier in \secref{motiv}). Since $G$ is the log partition function of exponential family, {\bf it is the log-partition function that determines the geometry of the space}. We emphasize that divergence is measured from the parameter being estimated to other parameters $\th_i$(s), as shown in \figref{conjugate}.
%HAL: can you put in an example figure?

\begin{example}{{\bf (1-D Gaussian)}}
The exponential family representation of a 1-d Gaussian is $p(x;\th) = \frac{1}{\sqrt{2\pi \sigma^2}}\exp(-\frac{(x-a)^2}{2\sigma^2})$ with $\th=\frac{a}{\sigma^2}$ and $G(\th) = \frac{\sigma^2}{2}\th^2$ whose ML estimation is just $(\grad G)^{-1}(\mu) = \frac{\mu}{\sigma^2}$ which gives $a=\mu = \frac{1}{n} \sum_i x_i$  i.e. data mean.
\end{example}

\begin{example}{{\bf (1-D Bernoulli)}}
The exponential family representation of a Bernoulli distribution $p = a^x (1-a)^{1-x}$ is the distribution with $\th = \log \frac{a}{1-a}$ with $G(\th) = \log (1+e^\th)$ whose ML estimation is $(\grad G)^{-1}(\mu) = \log \frac{\mu}{1-\mu}$. Comparing it with $\th$ gives $a = \mu = \frac{1}{n} \sum_i x_i$ which is the estimated probability of the event in $n$ trials.
\end{example}

\subsection{Conjugate Prior in the form of Bregman Divergence}
\label{sec:priorGeom}
We now give an expression similar to the likelihood for the conjugate prior:
{\small
\begin{align}
\log p(\th|\alpha, \beta) = \log m(\alpha,\beta) + \beta (\innerprod{\th, \frac{\alpha}{\beta}} - G(\th))
\label{conjugate}
\end{align}
}
\eqref{conjugate} can be written in the form of Bregman divergence by a direct comparison to \eqref{expdist}, replacing $x$ with $\alpha/\beta$.
{\small
\begin{align}
& \log p(\th|\alpha, \beta) = \log m(\alpha,\beta) + \beta \left(F\left(\frac{\alpha}{\beta}\right) - B_F\left(\frac{\alpha}{\beta} \rVert \grad G(\th)\right)\right)
\label{priorbregman}
\end{align}
}
The expression for the joint probability of data and parameters is given by:
{\small
\begin{align*}
 \log p(x,\th|\alpha, \beta) & = \log p_o(x)  + \log m(\alpha,\beta) + F(x) +  \beta F\left(\frac{\alpha}{\beta}\right) \nonumber  \\
& \hspace{4em}  - \left(B_F(x \rVert \grad G(\th)) + \beta B_F\left(\frac{\alpha}{\beta} \rVert \grad G(\th)\right)\right)
\end{align*}
}
Combining all terms that do not depend on $\th$:
\begin{align}
 \log p(x,\th|\alpha, \beta) &  =  \text{const} - B_F(x \rVert \mu ) - \beta B_F\left(\frac{\alpha}{\beta} \rVert \mu\right)
\label{mapproblem}
\end{align}

\subsection{Geometric Interpretation of Conjugate Prior}
\label{sec:conjugateIter}
In this section we give a geometric interpretation of the term
$B_F(x \rVert \mu) + \beta B_F(\frac{\alpha}{\beta} \rVert \mu)$ from
\eqref{mapproblem}.

\begin{theorem}[MAP as Bregman median]
 Given a set $\c{X}$ of $n$ i.i.d examples drawn from the exponential family distribution with the log partition function $G$ and a conjugate prior as in \eqref{priorbregman}, MAP estimation of parameters is $\hat{\th}_{MAP} = \hat{\mu}_{MAP}^*$ where $\hat{\mu}_{MAP}$ solves the following problem:
{\small
\begin{align}
\hat{\mu}_{MAP} = \min_{\mu \in \c{M}} \sum_{i=1}^n B_F(x_i \rVert \mu) + \beta B_F\left(\frac{\alpha}{\beta} \rVert \mu\right)
\label{jointmin}
\end{align}
}
which admits the following solution:
{\small
\begin{align*}
\mu = \frac{\sum_{i=1}^n x_i + \alpha}{n+\beta}
\end{align*}
}
\end{theorem}
\begin{proof}
MAP estimation by definition maximizes \eqref{mapproblem} for all data points $\c{X}$ which is equivalent to minimizing $B_F(x_i \rVert \mu) + \beta B_F(\frac{\alpha}{\beta} \rVert \mu)$. One can expand this expression using \eqref{bregman} and use conditions $F(\grad G(\th)) + G(\th) = \grad G(\th)\th $ and $\grad F(\th) = (\grad G)^{-1}(\th)$ to obtain the desired solution.
\end{proof}

The above solution gives a natural interpretation of MAP estimation. One can think of prior as $\beta$ number of extra points at position $\alpha/\beta$. $\beta$ works as the effective sample size of the prior .
{\small
\begin{align}
 \mu = \frac{\sum_{i=1}^n x_i + \sum_{i=1}^{\beta} \frac{\alpha}{\beta}}{n+\beta}
\label{closedsol}
\end{align}
}
The expression \eqref{jointmin} is analogous to \eqref{mlmeanspace} in the sense that both are defined in the dual space $\c{M}$. One can convert \eqref{jointmin} into an expression similar to \eqref{mlindual} in the dual space which is again a Bregman median problem in the parameter space.
{\small
\begin{align}
\hat{\th}_{MAP} = \min_{\th \in \Theta} \sum_{i=1}^n B_G(\th \rVert \th_i) + \sum_{i=1}^{\beta} B_G\Big(\th \rVert (\frac{\alpha}{\beta})^*\Big)
\label{jointmin2}
\end{align}
}
here $(\frac{\alpha}{\beta})^* \in \Theta$ is dual of $\frac{\alpha}{\beta}$. Above problem is Bregman median problem of $n+\beta$ points, $\{\th_1, \th_2 \ldots \th_n,\underbrace{(\alpha/\beta)^*,\ldots,(\alpha/\beta)^*}_{\beta \text{ times}}\}$, as shown in \figref{conjugate} (left).

A geometric interpretation is also shown in \figref{conjugate}. When the prior is conjugate to the likelihood, they both have the same log-partition function (\figref{conjugate}, left). Therefore they induce the same Bregman divergence. Having the same divergence means that distances from $\th$ to $\th_i$ (in likelihood) and the distances from $\th$ to $(\alpha/\beta)^*$ are measured with the same divergence function, yielding the same geometry for both spaces.

It is easier to see using the median formulation of MAP estimation problem that one must choose a prior that is conjugate.  If one chooses a conjugate prior, then the distances among all points are measured using the same function. It is also clear from \eqref{closedsol} that in the conjugate prior case, point induced by the conjugate prior $(\alpha/\beta)^*$ behaves as a sample point. A median problem over a space that have different geometries is an ill-formed problem, as discussed further in the next section.

\subsection{Geometric Interpretation of Non-conjugate Prior}
\label{sec:nonConjugateIter}
We derived expression \eqref{jointmin2} because we considered the prior conjugate to the likelihood function. Had we chosen a non-conjugate prior with log-partition function $Q$, we would have obtained:
{\small
\begin{align}
\hat{\th}_{ML} = \min_{\th \in \Theta} \sum_{i=1}^n B_G(\th \rVert \th_i) + \sum_{i=1}^{\beta} B_Q\left(\th~ \rVert \left(\frac{\alpha}{\beta}\right)^*\right).
\label{jointminnc}
\end{align}
}

\noindent
Here $G$ and $Q$ are different functions defined over $\Theta$. Since these are the functions that define the geometry of the space parameter, having $G \neq Q$ is equivalent to consider them as being defined over different (metric) spaces. Here, it should be noted that distance between sample point ($\th_i$)and the parameter $\th$ is measured using the Bregman divergence $B_G$. On the other hand, the distance between the point induced by the prior $(\alpha/\beta)^*$ and $\th$ is measured using the divergence function $B_Q$.  This means that $(\alpha/\beta)^*$ can \emph{not} be treated as one of the sample points. This tells us that, unlike the conjugate case, belief in the non-conjugate prior can not be encoded in the form of the sample
points. 
%
%Although it might be possible to find a solution to the above expression, it does make sense to find the median of a set of points when a subset uses a different functions to measure the divergence.

Another problem with considering a non-conjugate prior is that dual space of $\Theta$ under different functions would be different.  Thus, one will not be able to find the alternate expression for \eqref{jointminnc} equivalent to \eqref{jointmin}, and therefore not be able to find the closed-form expression similar to \eqref{closedsol}. This tells us why non-conjugate does not give us a closed form solution for $\hat{\th}_{MAP}$.

A pictorial representation of this is also shown in \figref{conjugate}. Note that, unlike the conjugate case, in the non-conjugate case, the data likelihood and the prior both belong to different spaces.

\begin{figure}[tc]
 \begin{center}
  \includegraphics[width=0.80 \textwidth]{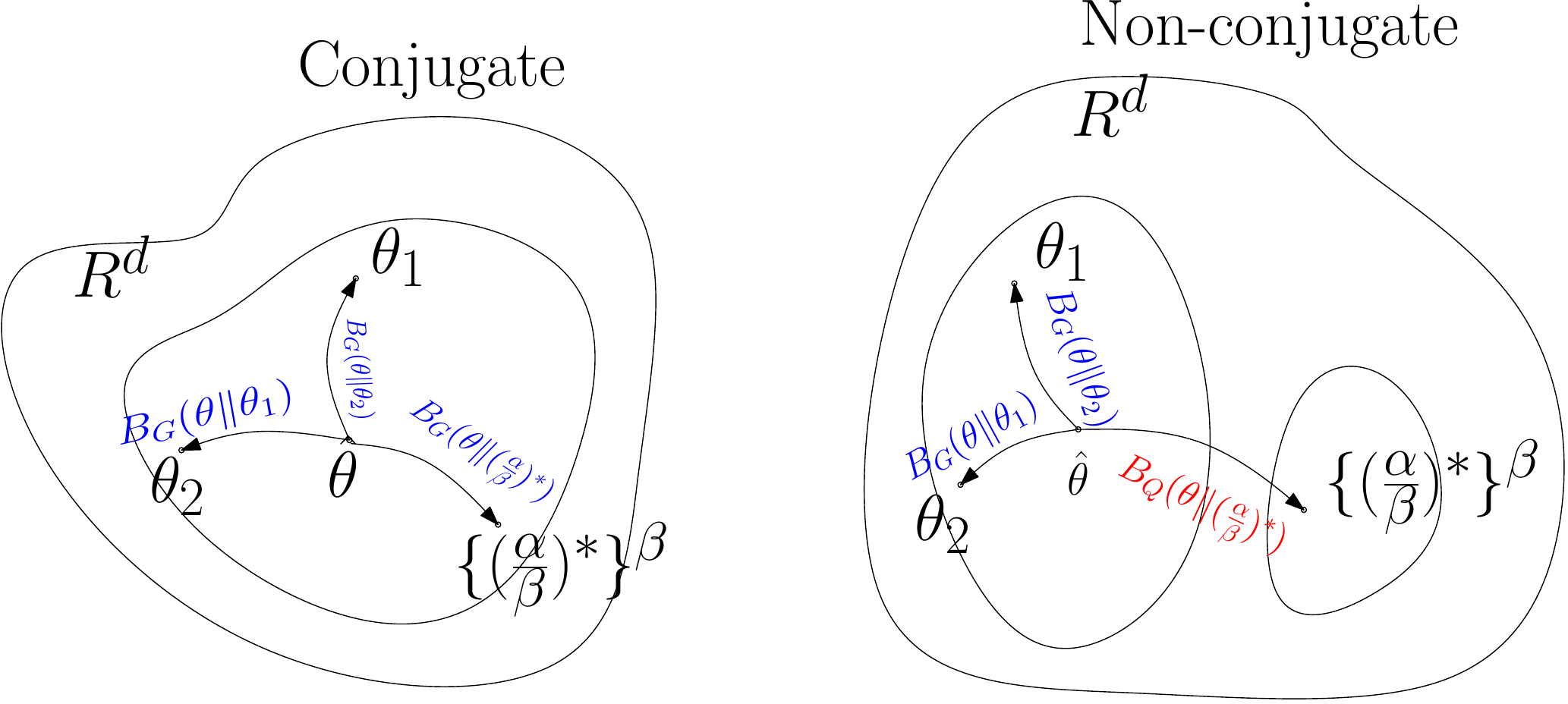}
  \caption{{\small\textsf{Prior in the conjugate case has the same geometry as the likelihood while in the non-conjugate case, they have different geometries.}}}
  \label{fig:conjugate}
 \end{center}
\end{figure}

%%%%%%%%%%%%%%%%%%%%%%%%%%%%%%%%%%%%%%%%%%%%%%%%%%%%%%%%%%%%%%%
\section{Information Geometric View}
\label{sec:infogeom}
In this section, we show the appropriateness of the conjugate prior from the information geometric angle. In information geometry, $\Th$ is a statistical manifold such that each $\th \in \Th$ defines a probability distribution. This statistical manifold has an inherent geometry, given by a \emph{metric} and an \emph{affine connection}. One natural metric is the Fisher information metric because of its many attractive properties: it is Riemannian and is invariant under reparameterization (for more details refer \cite{amarinagaoka}). 

In exponential family distributions, the Fisher metric $M(\th)$ is induced by the KL-divergence $KL(\cdot\|\th)$, which is equivalent to the Bregman divergence defined by the log-partition function. Thus, it is the log-partition function $G$ that induces the Fisher metric, and therefore determines the {\it natural} geometry of the space. It justifies our earlier argument of choosing the log-partition function to define the geometry. Now if we were to treat the prior as a point on the statistical manifold defined by the likelihood model, Fisher information metric on the point given by the prior must be same as the one defined on likelihood manifold. This means that prior must have the same log-partition function as the likelihood i.e., it must be conjugate.

\subsection{Generalization to $\alpha$-affine manifold}
\label{sec:alphaFamily}
Not all probability distributions belong to the exponential family (although many do).  A broader family of distributions is the ``$\alpha$-family'' \cite{amarinagaoka}.  Although a full treatment of this family is beyond the scope of the work, we briefly discuss an extension of our results to the $\alpha$-family.  An $\alpha$-family distribution
is defined as:
{\small
\begin{align*}
\log p_{\alpha}(x;\th) = \left\{
 \begin{array}{ll}
\frac{2}{1-\alpha} p(x;\th)^{(1-\alpha)/2} \quad &\alpha \ne 1\\
   \log p(x;\th) \quad \quad  \quad &\alpha=1
   \end{array} \right.
\end{align*}
}
%HAL: should there be a log in the first line of the array?  If not, perhaps explain why not?  I assume it's because you obtain the second line as the limiting case of the first line as alpha approaches one?
where $p(x;\th)$ defined as in \eqref{expdist}. Note that the exponential family is a special case of $\alpha$-family for $\alpha=1$.

MAP estimation of the parameters in the exponential family can be cast as a median problem, where an appropriate Bregman divergence is used to define the geometry.  In other words, for exponential family, a Bregman-median problem naturally arose as an estimation method.

By using an appropriately defined, ``natural,'' divergence for the $\alpha$-family, one can actually obtain a similar result for this broader family of distrubtions.  Using such a natural divergence, one can also define a ``conjugate prior'' for the $\alpha$-family.  Zhang et al.~\cite{zhang04divegence} shows that such a natural divergence exist for $\alpha$-family and is given by:
{\small
\begin{align*}
D_G^{\alpha}(\th_1,\th_2) = \frac{4}{1-\alpha^2}\left( \frac{1-\alpha}{2} G(\th_1) + \frac{1+\alpha}{2} G(\th_2) -  G\left(\frac{1-\alpha}{2}\th_1 + \frac{1+\alpha}{2}\th_2 \right)\right)
\end{align*}
}
Like the exponential family, this divergence also induces the Fisher information metric. 
%and the dualistic structure\footnote{This duality refer to the representation duality like the duality induced by the Bregman divergence, not to the referential duality. For more details on this, see \cite{zhang04divegence}}.

%%%%%%%%%%%%%%%%%%%%%%%%%%%%%%%%%%%%%%%%%%%%%%%%%%%%%%%%%%%%%%%
\section{Hybrid model}
\label{sec:hybrid}
In this section, we show an application of our analysis to a common supervised and semi-supervised learning framework.  In particular, we consider a generative/discriminative hybrid model \cite{arvind09hybrid,Druck2007,Minka2006} that has been shown to be successful in many application.  
%and show why it is not appropriate to
%couple two arbitrary models using a Gaussian (or non-conjugate)
%prior. We do this by taking a geometric view of the likelihood and the
%prior, and show that choosing a prior conjugate to the generative
%model is appropriate beyond the reasons given in
%\cite{arvind09hybrid}. In particular, we give much simpler analysis
%based on the geometric intuition, which naturally gives the expression
%and the hyperparameters of the conjugate prior.

The hybrid model is a mixture of discriminative and generative models,
%\footnote{In generative models, data is assumed to be generated   through some underlying process and goal is to model this process   while in discriminative model, there is no underlying process and   goal is to find the boundary that discriminates the data}
 each of which has its own separate set of parameters. These two sets of parameters (hence two models) are combined using a prior called the {\it coupling   prior}.  Let $ p(y|\x,\thd)$ be the discriminative component, $p(\x,y|\thg)$ be the generative component and $p(\thd,\thg)$ be the prior that couples discriminative and generative components. The joint likelihood of the data and parameters is:
{\small
\begin{align}
\label{rawmodel}
p(\x,y,\thd,\thg)&=p(\thg,\thd)p(y|\x,\thd)p(\x|\thg) \\
&=p(\thg,\thd)p(y|\x,\thd)\sum_{y'}p(\x,y'|\thg) \nonumber
\end{align}
}
Here $\thd$ is a set of discriminative parameters, $\thg$ a set of generative parameters, and $p(\thg,\thd)$ provides the natural coupling between these two sets of parameters.

The most important aspect of this model is the \emph{coupling prior} $p(\thg,\thd)$, which {\it interpolates} the hybrid model between two extremes: fully generative when the prior forces $\thd=\thg$, and fully discriminative when the prior renders $\thd$ and $\thg$ independent. In non-extreme cases, the goal of the coupling prior is to encourage the generative model and the discriminative model to have similar parameters. It is easy to see that this effect can be induced by many functions. One obvious way is to {\it linearly} interpolate them as done by \cite{Minka2006,Druck2007} using a Gaussian prior (or the Euclidean distance) of the following form:
{\small
\begin{equation}
p(\thg,\thd) \varpropto  \exp \left( -\lambda \norm{\thg - \thd}^2 \right)
\label{guassprior}
\end{equation}
}
where, when  $\lambda=0$, model is purely discriminative while for $\lambda=\infty$, model is purely generative. Thus $\lambda$ in the above expression is the interpolating parameter, and is same as the $\gamma$ in \secref{motiv}. Note that log of prior is nothing but the squared Euclidean distance between two sets of parameters.

It has been noted multiple times \cite{Bouchard2007,arvind09hybrid} that a Gaussian prior is not always appropriate, and the prior should instead be chosen according to models being considered.  Agarwal et al.~\cite{arvind09hybrid} suggested using a prior that is conjugate to the generative model. Their main argument for choosing the conjugate prior came from the fact that this provides a closed form solution for the generative parameters and therefore is mathematically convenient. We will show that it is more than convenience that makes conjugate prior appropriate.  We show that choosing a non-conjugate prior is not only not convenient but also not appropriate. Moreover, our analysis does not assume anything about the expression and the hyperparameters of the prior beforehand, rather derive them automatically.

\subsection{Generalized Hybrid Model}
In order to see the effect of the geometry, we first present the generalized hybrid model for distributions that belong to the exponential family and present them in form of Bregman divergence. Following the expression used in \cite{arvind09hybrid}, the generative model can be written as:
{\small
\begin{align}
p(\x,y|\thg)=h(\x,y)\exp(\langle\thg, T(\x,y)\rangle-G(\thg))
\label{expgen}
\end{align}
}
where $T(\cdot)$ is the potential function similar to $\phi$ in \eqref{expdist}, now only defined on $(\x,y)$.

Let $G^*$ be the dual function of $G$; the corresponding Bregman divergence (retaining only the terms that depend on the parameter $\theta$) is given by:
{\small
\begin{align}
B_{G^*}\left((\x,y) \rVert \grad G(\thg)\right).
\label{expgendiv}
\end{align}
}
Solving the generative model independently reduces to choosing a $\thg$ from the space of all generative parameters $\Theta_g$ which has a geometry defined by the log-partition function $G$. Similarly to the generative model, exponential form of the discriminative model is given as:
\begin{align}
p(y|\x,\thd)=\exp(\langle \thd, T(\x,y) \rangle -M(\thd,\x))
\label{gendisc}
\end{align}
Importantly, the sufficient statistics $T$ are the \emph{same} in the generative and discriminative models; such generative/discriminative pairs occur naturally: logistic regression/naive Bayes and hidden Markov models/conditional random fields are examples.  However, observe that in the discriminative case, the log partition function $M$ depends on both $\x$ and $\thd$ which makes the analysis of the discriminative model harder. Unlike the generative model, one does not have the explicit form of the log-partition function $M$ that is independent of $\x$.  This means that the discriminative component
\eqref{gendisc} can not be converted into an expression like \eqref{expgendiv}, and MLE problem can not be reduced to the Bregman median problem like the one given in \eqref{mlindual}.

\subsection{Geometry of the Hybrid Model}
\begin{figure}[tc]
 \begin{center}
  \includegraphics[width=0.40 \textwidth]{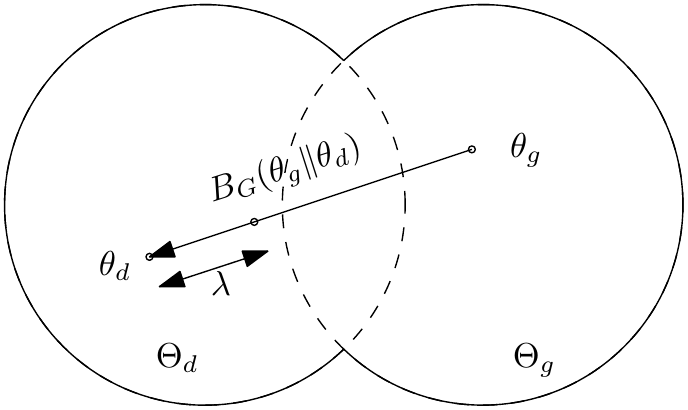}
  \caption{{\small\textsf{Parameters $\thd$ and $\thg$ are interpolated using the Bregman divergence}}}
  \label{fig:interpolate}
 \end{center}
\end{figure}

We simplify the analysis of the hybrid model by writing the discriminative model in an alternate form. This alternate form makes obvious the underlying geometry of the discriminative model.  Note that the only difference between the two models is that discriminative model models the conditional distribution while generative model models the joint distribution. We can use this observation to write the discriminative model in the following alternate form using the expression $p(y|x,\th) = \frac{p(y,x|\th)}{\sum_{y'} p(y'x|\th)}$ and \eqref{expgen}:
{\small
\begin{align}
p(y|x,\thd)=\frac{h(\x,y)\exp(\langle\thd, T(\x,y)\rangle-G(\thd))}{\sum_{y'}h(\x,y')\exp(\langle\thd, T(\x,y')\rangle-G(\thd))}
\label{expdisc}
\end{align}
}
%HAL: I think you need to explain a bit why this is true... namely, how did M turn in to G?  Why does G need to be there at all?
Denote the space of parameters of the discriminative model by $\Th_d$. It is easy to see that geometry of $\Th_d$ is defined by $G$ since function $G$ is defined over $\thd$.  This is same as the geometry of the parameter space of the generative model $\Th_g$. Now let us define a new space $\Th_H$ which is the {\it affine} combination of $\Th_d$ and $\Th_g$.  Now, $\Th_H$ will have the same geometry as $\Th_d$ and $\Th_g$ i.e., geometry defined by $G$. Now the goal of the hybrid model is to find a $\th \in \Th_H$ that maximizes the likelihood of the data under the hybrid model. These two spaces are
shown pictorially in \figref{interpolate}.

% Similarly, Bregman form of the discriminative model is:
% \begin{equation}
% B_{A^*}(\x,y \rVert\grad  A(\thd)) - f(x;\thd),
% \label{expgendiv}
% \end{equation}
% where $f$ is some function computed from the denominator in \eqref{expdisc}.

\subsection{Prior Selection}
As mentioned earlier, the coupling prior is the most important part of the hybrid model, which controls the amount of coupling between the generative and discriminative models. There are many ways to do this, one of which is given by \cite{Minka2006,Druck2007}. By their choice of Gaussian prior as coupling prior, they implicitly couple the discriminative and generative parameters by the squared Euclidean distance. We suggest coupling these two models by a general prior, of which the Gaussian prior is a special case.

\subsubsection{Bregman Divergence and Coupling Prior:}
Let a general coupling be given by $B_S(\thg\|\thd)$. Notice the direction of the divergence. We have chosen this direction because prior is induced on the generative parameters, and it is clear from \eqref{jointmin2} that parameters on which prior is induced, are placed in the first argument in the divergence function. The direction of the divergence is also shown in \figref{interpolate}.

Now we recall the relation \eqref{priorbregman} between the Bregman divergence and the prior. Ignoring the function $m$ (this is consumed in the measure defined on the probability space) and replacing $\grad{G(\th)}$ by $\th^*$, we get the following expression:
{\small
\begin{align}
\log p(\th_g|\alpha, \beta) = \beta (F(\frac{\alpha}{\beta}) - B_F(\frac{\alpha}{\beta} \rVert \th_g^*))
\label{priorbregman1}
\end{align}
}
Now taking the $\alpha=\lambda \th_d^*$ and $\beta=\lambda$, we get:
{\small
\begin{align}
\log p(\th_g|\lambda \th_d^*, \lambda) & = \lambda (F(\th_d^*) - B_F(\th_d^* \rVert \th_g^*))\\
p(\th_g|\lambda \th_d^*, \lambda) & = \exp(\lambda (F(\th_d^*))) \; \exp(- \lambda B_F(\th_d^* \rVert \th_g^*))
\label{priorbregman2}
\end{align}
}
For the general coupling divergence function $B_S(\thg\|\thd)$, the corresponding coupling prior is given by:
{\small
\begin{align}
\exp(- \lambda B_{S^*}(\th_d^* \rVert \th_g^*)) = \exp(-\lambda (F(\th_d^*))) \; p(\th_g|\lambda \th_d^*, \lambda)
\label{priorbregman2}
\end{align}
}
The above relationship between the divergence function (left side of the expression) and coupling prior (right side of the expression) allows one to define a Bregman divergence for a given coupling prior and vise versa. 

\subsubsection{Coupling Prior for the Hybrid Model:}
We know that that the geometry of the space underlying Gaussian prior is just Euclidean, which does not necessarily match the geometry of the likelihood space. The relationship between prior and divergence \eqref{priorbregman2} allows one to first define the appropriate geometry for the model, and then define the prior that respects this geometry. In hybrid model, this geometry is given by the log partition function $G$ of the generative model. This argument suggests to couple the hybrid model by the divergence of the form $B_G(\thg \|\thd)$. The coupling prior corresponding to this divergence function can be written using \eqref{priorbregman2} as:
{\small
\begin{align}
\exp(-\lambda B_G(\thg \| \thd)) = p(\thg | \lambda \thd^*, \lambda)\;\exp(-\lambda F(\th^*_d))
\end{align}
}
where $\lambda=[0,\infty]$ is the interpolation parameter, interpolating between the discriminative and generative extremes. In dual form, the above expression can be written as:
{\small
\begin{align}
\exp(-\lambda B_G(\thg \| \thd)) = p(\thg | \lambda \thd^*, \lambda)\;\exp(-\lambda G(\th_d)).
\end{align}
}

Here $\exp(-\lambda G(\thd))$ can be thought of as a prior on the discriminative parameters $p(\thd)$. In above expression, $\exp(-\lambda B_G(\thg \| \thd)) = p(\thg|\thg)p(\thd) $ behaves as a joint coupling prior $P(\thd,\thg)$ as originally expected in the model \eqref{rawmodel}. Note that hyperparameters of the prior $\alpha$ and $\beta$ are naturally derived from the geometric view of the conjugate prior. Here $\alpha = \lambda \th_d^*$ and $\beta=\lambda$.  

\subsubsection{Relation with Agarwal et al.:}
%HAL: This seems redundant
The prior we derived in the previous section turns out to be the exactly same as that proposed by Agarwal et al.~\cite{arvind09hybrid}, even though theirs was not formally justified.  In that work, the authors break the coupled prior $p(\thg,\thd)$ into two parts: $p(\thd)$ and $p(\thg|\thd)$. They then derive an expression for the $p(\thg|\thd)$ based on the intuition that the mode of $p(\thg|\thd)$ should be $\thd$. Our analysis takes a different approach by coupling two models with the Bregman divergence rather than prior, and results in the expression and hyperparameters for the prior same as in \cite{arvind09hybrid}.

The two analyses diverge here, however.  Our analysis derives the hyperparameters as: $\alpha = \lambda (\grad G)^{-1}(\th_d)$ and $\beta = \lambda$.  However, the expression of the hyperparameters provided by Agarwal et al.~\cite{arvind09hybrid} was: $\alpha = \lambda \grad G(\th_d)$ and $\beta = \lambda$.  Their derivation was the assumption that the mode of the coupling prior $p(\thg|\thd)$ should be $\thd$. However, in the conjugate prior $p(\th|\alpha,\beta)$, the mode is $\frac{\alpha}{\beta}$, and $\frac{\alpha}{\beta}$ behaves as the sufficient statistics for the prior.  These terms have come from the data space, \emph{not} from the parameter space. Therefore the mode of the coupling prior $p(\thg|\thd)$ should not be $\thd$, but rather the dual of $\thd$ which is $(\grad G)^{-1} (\thd) = \th_d^*$.  Therefore, $\alpha = \lambda \th_d^*$ and $\beta = \lambda$ and our model gives exactly
this.

\section{Related Work and Conclusion}
To our knowledge, there have been no previous attempts to understand Bayesian priors from a geometric perspective. One related piece of work \cite{Snoussi02informationgeometry} uses the Bayesian framework to find the best prior for a given distribution. It is noted that, in that work, the authors use the $\delta$-geometry for the data space and the $\alpha$-geometry for the prior space, and then show the different cases for different values $(\delta, \alpha)$. We emphasize that even though it is possible to use different geometry for the both spaces, it always makes more sense to use the same geometry. As mentioned in {\it remark} 1 in \cite{Snoussi02informationgeometry}, useful cases are obtained only when we consider the same geometry.

We have shown that by considering the geometry induced by a likelihood function, the natural prior that results is exactly the conjugate prior. We have used this geometric understanding of conjugate prior to derive the coupling prior for the discriminative/generative hybrid model. Our derivation naturally gives us the expression and the hyperparameters of this coupling prior. Like the hybrid model, this analysis can be used to give the much simpler geometric interpretations of many models, and to extend the existing results to other models, i.e. we have used this analysis to extend the geometric formulation of MAP problem for the exponential family to $\alpha$-family.

{\small
\bibliography{ref}
\bibliographystyle{splncs03}
}
\end{document}